\documentclass{article}
\pdfoutput=1

\PassOptionsToPackage{numbers, compress}{natbib}

    \PassOptionsToPackage{numbers, compress}{natbib}
\usepackage{amsmath,amsfonts,bm}

\def\eqref#1{equation~\ref{#1}}
\def\1{\bm{1}}

\DeclareMathAlphabet{\mathsfit}{\encodingdefault}{\sfdefault}{m}{sl}
\SetMathAlphabet{\mathsfit}{bold}{\encodingdefault}{\sfdefault}{bx}{n}

\def\sR{{\mathbb{R}}}

\newcommand{\R}{\mathbb{R}}

\newcommand{\KL}{{\mathrm{KL}}}
\newcommand{\FD}{{\mathrm{FD}}}
\newcommand{\MFD}{\widetilde{\mathrm{FD}}}

\newcommand{\X}{\mathcal{X}}

\DeclareMathOperator{\Normal}{\mathcal{N}}

\newcommand{\tp}{\tilde{p}}
\newcommand{\tq}{\tilde{q}}

\newcommand{\norm}[1]{\left\lVert#1\right\rVert}

\DeclareMathOperator{\Tr}{Tr}

\renewcommand{\MFD}{\mathrm{MFD}}
\usepackage{centernot}

    \usepackage[final]{score_workshop}

\usepackage[utf8]{inputenc} % allow utf-8 input
\usepackage[T1]{fontenc}    % use 8-bit T1 fonts
\usepackage{hyperref}       % hyperlinks
\usepackage{url}            % simple URL typesetting
\usepackage{booktabs}       % professional-quality tables
\usepackage{amsfonts}       % blackboard math symbols
\usepackage{nicefrac}       % compact symbols for 1/2, etc.
\usepackage{microtype}      % microtypography
\usepackage{xcolor}         % colors
\usepackage{float}
\usepackage{amsmath}
\usepackage{amsthm}
\usepackage{graphicx}
\usepackage{subcaption}
\usepackage{tikz}
\usetikzlibrary{bayesnet}
\usetikzlibrary{matrix}
\usepackage{nicematrix}
\usepackage{stackengine}
\usepackage{wrapfig}
\usepackage{algorithm}
\usepackage{centernot}
\usepackage{algpseudocode}

\newenvironment{talign*}
{\csname align*\endcsname}
{\endalign}
\newenvironment{talign}
{\align}
{\endalign}

\usetikzlibrary{arrows, fit, matrix, positioning, shapes, backgrounds}
\usepackage{cancel}
\newtheorem{prop}{Proposition}
\newtheorem{theorem}{Theorem}
\newtheorem{lemma}[theorem]{Lemma}
\theoremstyle{definition}

\newcommand{\iid}{\textit{i.i.d.\ }}

\title{Towards Healing  the  Blindness of  Score Matching}

\author{%
 Mingtian Zhang\\
 University College London\\
  \texttt{m.zhang@cs.ucl.ac.uk} \\
  \And 
  Oscar Key\\ 
   University College London\\
  \texttt{o.key@cs.ucl.ac.uk} \\
  \And 
    Peter Hayes\\ 
   University College London\\
  \texttt{p.hayes@cs.ucl.ac.uk} \\
  \AND 
    David Barber \\ 
   University College London\\
  \texttt{david.barber@ucl.ac.uk}\\
  \And 
    Brooks Paige\\ 
   University College London\\
  \texttt{b.paige@ucl.ac.uk} \\
  \And 
    Fran\c{c}ois-Xavier Briol\\ 
   University College London\\
  \texttt{f.briol@ucl.ac.uk} 
}
\begin{document}

\maketitle

\begin{abstract}
Score-based divergences  have been widely used in machine learning and statistics applications. Despite their empirical success, a blindness problem has been observed when using these for multi-modal distributions. In this work, we  discuss the blindness problem and propose a new family of divergences that can mitigate the blindness problem. We illustrate our proposed divergence in the context of density estimation and report improved performance compared to traditional approaches.
\end{abstract}

\section{Introduction}
Score-based divergences such as the Fisher Divergence (FD; also known as score-matching divergence)~\cite{hyvarinen2005estimation,hyvarinen2007some} and Kernel Stein Discrepancy (KSD)~\cite{liu2016kernelized,chwialkowski2016kernel} are widely used in machine learning and statistics~\cite{anastasiou2021stein,song2021train}. Their main advantage is that the score function, a derivative of a log-density, can be evaluated without knowledge of the normalization constant of the density and  can be applied to problems where other classical divergences (e.g. KL divergence) are intractable.
% These divergences can hence be applied to a wide range of problems where other classical divergences, such as the KLullback-Leibler 
% divergence or Wasserstein distance, would be intractable.  #\mt{modify for extra spaces}
Unfortunately, this advantage can also be a curse in certain scenarios because the score function only provides local information about the slope of a density, but ignores more global information such as the importance of a point relative to another. This has led to a blindness problem in many applications of score-based methods where the densities are multi-modal, including in density estimation \cite{wenliang2019learning,song2019generative,jolicoeur2020adversarial}, MCMC convergence diagnosis~\cite{gorham2019measuring}, Bayesian inference~\cite{Matsubara2022,d2021annealed};
see \cite{wenliang2020blindness} for a detailed discussion.

To illustrate this problem, we recall the definition of FD and an example from \cite{wenliang2020blindness}. Given two distributions with differentiable densities $p$ and $q$ supported on a common domain $\X\subseteq \sR^d$, the FD is
\begin{talign}
  \FD(p||q) = \frac{1}{2}\int_{\X}p(x)||s_p(x)-s_q(x)||^2_2  dx,
\end{talign}
where we denote by  $s_p(x)\equiv\nabla_x\log p(x)$ and $s_q(x)\equiv \nabla_x\log q(x)$ the score functions of $p$ and $q$ respectively.
The classic sufficient conditions~\cite{hyvarinen2005estimation,barp2019minimum} for the FD to be a valid statistical divergence (i.e.
$\FD(p||q)= 0 \Leftrightarrow p=q$)  are: (i) $p$ and $q$ are differentiable with support $\X=\sR^d$ 
and
(ii) $s_p,s_q$ are square integrable, i.e. $s_p-s_q\in L^2(p)$, where we denote $f\in L^2(p)\equiv \int_\X ||f(x)||_2^2 p(x)dx<\infty$. The blindness problem of the FD can be illustrated through the following example due to \cite{wenliang2020blindness}. Let $p$ and $q$ be a mixtures with the same components but different mixing weights:
\begin{talign}
p(x)=\alpha_p g_1(x)+(1-\alpha_p)g_2(x), \quad q(x)=\alpha_q g_1(x)+(1-\alpha_q)g_2(x),\label{eq:toy:mixture}
\end{talign}
where $\alpha_p\neq \alpha_p$, and 
$g_1,g_2$ are Gaussian densities 
 with variance $\sigma^2$ and means $-\mu$ and $\mu$ respectively. Then $\mathrm{FD}(p||q)\rightarrow 0$ when $ \mu/\sigma^2\rightarrow\infty$ regardless of the mixture proportions $\alpha_p$ and $\alpha_q$. To build intuition, we let  $\mu=5$, $\sigma=1$, $\alpha_p=0.2, \alpha_q=0.8$  and plot the densities and score functions of $p,q$ in Figure~\ref{fig:desnities} and \ref{fig:scores}. We can find the two distributions are very different but their scores are only different around $x=0$, which has a negligible density value under  $p$. We then fix $\alpha_p=0.2$ and plot the $\FD(p||q)$ as a function of $\alpha$ in Figure \ref{fig:proportions}. Here we see the FD is $0$ constant function, which shows the FD is `blind' to the value of the mixture weight. See \cite{Matsubara2022} for a similar example for discrete $\X$.
\begin{figure}[t]
\centering
\begin{subfigure}[b]{0.32\textwidth}
     \centering
    \includegraphics[width=\textwidth]{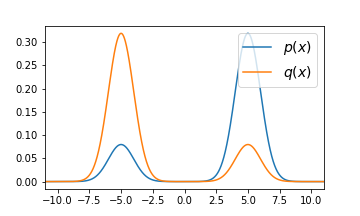}
    \caption{Densities of $p$ and $q$}
    \label{fig:desnities}
\end{subfigure}
  \begin{subfigure}[b]{0.32\textwidth}
     \centering
    \includegraphics[width=\textwidth]{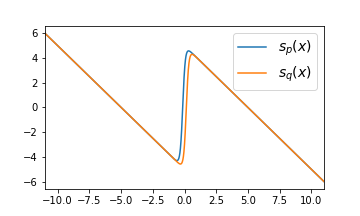}
    \caption{Score functions of $p$ and $q$}
    \label{fig:scores}
\end{subfigure}
  \begin{subfigure}[b]{0.32\textwidth}
     \centering
    \includegraphics[width=\textwidth]{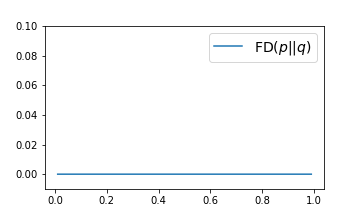}
    \caption{$\FD(p||q)$ with different $\alpha_q$}
    \label{fig:proportions}
\end{subfigure}
\caption{We plot the densities and score functions  of distributions $p$ and $q$ in Figure (a) and (b). Figure (c) shows $\FD(p||q)$ with $\alpha_p=0.2$ and $\alpha_q$ varies from $0.01$ to $0.09$ with a grid size $0.01$. \label{fig:blind}}
\vspace{-0.3cm}
\end{figure}

\section{Understanding the Blindness Problem}
In the example above, blindness is a numerical problem since the problem occurs despite the fact that the FD is a divergence in that case (i.e. $\FD(p||q)=0\Leftrightarrow p=q$ since (i) and (ii) are satisfied). When $\mu/\sigma^2\rightarrow\infty$, although the Gaussian distributions still have the same support, the regions that contain most of the mass of $g_1$ and $g_2$ tend to be disjoint, which creates numerical issues. 
However, the blindness problem is not simply a numerical problem, as illustrated in the following example.

Consider the case where $p$ and $q$ are mixtures whose identical components have disjoint supports.
For example, let $g_1$ and $g_2$ in Equation \ref{eq:toy:mixture} have disjoint support sets $\X_1,\X_2\subseteq \sR^d$ respectively with $\X_1\cap \X_2=\emptyset$. Then, $g_2(x')=\nabla_x g_2(x')=0$ for $x'\in \X_1$ and $g_1(x')=\nabla_x g_1(x')=0$ for $x'\in \X_2$.  In this case, the FD is independent of $\alpha_q$ (see Appendix \ref{app:derivation} for a derivation):
\begin{talign}
\resizebox{0.9\hsize}{!}{$%
    \mathrm{FD}(p||q)=\frac{\alpha_p}{2} \int_{\X_1} g_1(x)||s_{g_1}(x)-s_{g_1}(x)||_2^2 dx   +\frac{1-\alpha_p}{2} \int_{\X_2} g_2(x)||s_{g_2}(x)-s_{g_2}(x)||_2^2 dx=0.$}\label{eq:disjoint}
\end{talign}
Therefore, the FD is not a valid divergence here since $\mathrm{FD}(p||q)=0\centernot\Rightarrow p=q$. This example guides us to further study the topology properties of the distributions' support required by the FD. We first extend the Fisher divergence to distributions that have support on the connected space.

% The disjoint support components assumption also violates one of the sufficient conditions of FD that distributions should have support $\sR^d$ and the density functions are differentiable on $\sR^d$. \fxb{I am not 100\% sure about this. In general, you could have a case where the support of both $p$ and $q$ is the whole of $\R^d$. This is because the support of a mixture is just the union of the support of individual components, so you would just choose $g_1$ and $g_2$ with support $(-\infty,0)$ and $[0,\infty)$. Or am I missing something?}\mt{Yeah, you are right, I guess adding differentiable will fix the problem?}
% In general, this phenomenon holds for distributions contain any number of disjoint supported components, which is stated in the following proposition.
\begin{theorem}[FD on a connected set]
    Assume two  distributions (i) have  differentiable densities $p$ and $q$ with support on a common open connected set $\X\subseteq \sR^d$ and (ii) $s_p-s_q\in L^2(p)$. Then, the FD is a valid divergence i.e. $\FD(p||q)=0 \Leftrightarrow p=q$. \label{theorem:connected}
\end{theorem}
See Appendix \ref{app:proof:connected} for a proof. Theorem~\ref{theorem:connected} generalizes the classic FD that is defined on distributions with $\X=\sR^d$~\cite{hyvarinen2005estimation,barp2019minimum} ($\sR^d$ is a special case of the connected set). Secondly, Theorem~\ref{theo:ill:defined} shows that \emph{connectedness} of the support is a \emph{necessary} condition to define a valid FD.
\begin{theorem}[FD is ill-defined on disconnected sets\label{theo:ill:defined}]
    Assume two  distributions have common support $\X$ consisting of disjoint sets. Then, the FD is not a valid divergence i.e. $\FD(p||q)=0 \centernot\Rightarrow p=q$.
\end{theorem}
% See Appendix for a proof.
% {\color{red}
% The way you prove it would obviously just be to say you can always write a distribution with disjoint support through a mixture of densities with components on each set, then use the proposition you have written below as an example which breaks the "divergence" definition. That being said, the message would be a lot easier to understand for the reader than if you keep the proposition below.
% }\mt{yeah, fully aggree}
% \fxb{I think this proposition could go to the appendix to save some space}\mt{yeah, ok}
% \mt{discuss KSD.}
See Appendix \ref{app:ill:defined} for a proof.  Intuitively, the score function only considers the local derivatives and contains no information of the global normalization constant. If the domain is disconnected, it cannot determine the mass allocation in different domains. This observation can also be extended to the KSD 
by viewing KSD as a kernelized FD~
% and is upper bounded by a  scaled FD~
\cite{liu2016kernelized,chwialkowski2016kernel}, see Appendix~\ref{app:ksd} for a detailed discussion.

\section{Healing the Blindness Problem with the Mixture Fisher Divergence}

In this section, we propose a new variant of the FD which is well-defined in the disconnected scenario. Consider a distribution with density $m$ with support $\X_m=\sR^d$ and define the mixtures
\begin{talign}
    \tilde{p}(x)=\beta p(x)+(1-\beta) m(x),\quad \tilde{q}(x)=\beta  q(x)+(1-\beta)  m(x),
\end{talign}
where $0<\beta<1$.
We then define the \emph{Mixture Fisher Divergence} (MFD)  as
\begin{talign}
    \MFD_{m,\beta} (p||q)\equiv \mathrm{FD}(\tilde{p}||\tilde{q}).
\end{talign}
Theorem~\ref{theo:mdf} shows the MFD is well-defined when $p$ and $q$ have support on a disconnected space.
\begin{theorem}[Validity of the MFD]\label{theo:mdf}
Consider two distributions with differentiable densities $p,q$  supported on  $\X_p,\X_q\subseteq\sR^d$ with $s_p,s_q\in L^2(p)$ and a differentiable density $m$ with support $\X_m=\sR^d, s_m\in L^2(p)$. Then MFD is a valid divergence, i.e. $\MFD(p||q)=0\Leftrightarrow \FD(\tilde{p}||\tilde{q}) \Leftrightarrow p=q$. 
\end{theorem}
See Appendix \ref{app:theorem} for a proof. For MFD,  we no longer require that $p,q$ have \emph{common connected support}, since $\X_m=\sR^d$ results in  $\tilde{p},\tilde{q}$ having connected support $\sR^d$\footnote{A weaker condition of $m$ can be obtained by requiring the supports of $\tilde{p},\tilde{q}$, which we denote as $\X_{\tilde{p}}, \X_{\tilde{q}}$, to be connected and $\X_{\tilde{p}}\subseteq \X_{\tilde{q}}$.  We here only study the stronger condition that $m(x)$ has support $\sR^d$ for simplicity.}. 
The requirements of $m(x)$ are mild and hold for simple choices of distribution e.g. a Gaussian.
To avoid the numerical problem mentioned in Section 1,  $m(x)$ should be chosen to effectively connect the different component distributions. As an example, for the toy problem described in Figure 1 with components $\Normal(-5,1)$ and $\Normal(-5,1)$ we can choose $\beta=0.5$ 
 and $m(x)=\Normal(0,9)$ that covers both components. Figure \ref{fig:mixture} shows the densities and their score functions for $\tilde{p},\tilde{q}$. We see that the score functions are different on the high-density region of $\tilde{p}$. Figure \ref{fig:mixture:proportions} also shows the minimal value of the $\MFD(p||q)$ is attained when $\alpha_q=\alpha_p$, which indicates that the proposed MFD heals the blindness problem in this  example.

% We also fix $\alpha_p=0.2$ and vary the value of $\alpha_q$ to plot the corresponding MFD in Figure , and
% the minimal value  of the $\MFD(p||q)$ estimations is attained when $\alpha_q=\alpha_p$, which indicates the proposed MFD with the chosen $m(x)$ can heal the blindness in this toy problem.

\begin{figure}[t]
\centering
\begin{subfigure}[b]{0.32\textwidth}
     \centering
    \includegraphics[width=\textwidth]{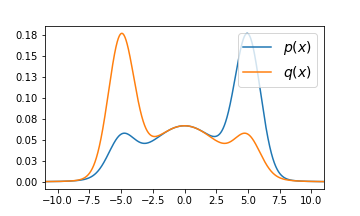}
    \caption{Densities of $\tilde{p}$ and $\tilde{q}$}
    \label{fig:mixture:desnities}
\end{subfigure}
  \begin{subfigure}[b]{0.32\textwidth}
     \centering
    \includegraphics[width=\textwidth]{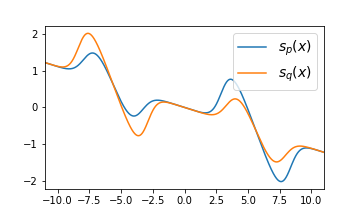}
    \caption{Score functions of $\tilde{p}$ and $\tilde{q}$}
    \label{fig:mixture:scores}
\end{subfigure}
  \begin{subfigure}[b]{0.32\textwidth}
     \centering
    \includegraphics[width=\textwidth]{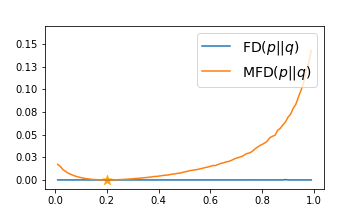}
    \caption{$\MFD(p||q)$ with different $\alpha_q$}
    \label{fig:mixture:proportions}
\end{subfigure}
\caption{We plot the  densities (a) and the score functions (b) of $\tilde{p}$ and $\tilde{q}$. Figure (c) shows $\MFD(p||q)$ with $\alpha_p=0.2$ and  $\alpha_q$ varies from $0.0$ to $1.0$ with a grid size $0.01$. The star mark shows the minima of the MFD is achieved when $\alpha_q=\alpha_p=0.2$, we also plot the original FD  for a comparison. \label{fig:mixture}}
\vspace{-0.3cm}
\end{figure}

\section{Density Estimation with Energy-based Models}
Given a dataset $\mathcal{X}_{\text{train}}=\{x_1,\cdots,x_N\}$ sampled \iid from a data distribution $p_d$  with support $\X_{p_d}\subseteq \sR^d$, we would like to learn a model $q_\theta$  to approximate $p_d$. We are interested in a family of models which can only be evaluated up to a normalization constant, e.g. an energy-based model $q_\theta(x)=e^{-f_\theta(x)}/Z(\theta)$, where $f_\theta$ is  a neural network and $Z(\theta)=\int e^{-f_\theta(x)}dx$. In this case, the standard Maximum Likelihood Estimation (MLE) is not applicable (since $Z(\theta)$ cannot be evaluated during training) and an
alternative form of the FD~\cite{hyvarinen2005estimation} can be applied (see Appendix~\ref{app:score:matching} for a derivation and additional assumptions)
% FD~\cite{hyvarinen2005estimation} can be applied to learn $\theta$.
% Since the score function of the underlying data distribution $s_{p_d}(x)$ is unknown in practice, an 
\begin{talign}
     \mathrm{FD}(p_d||q_\theta)=\frac{1}{2}\int_{\X_{p_d}} p_d(x)\left( ||s_{q_\theta}(x)||_2^2+2 \Tr(\nabla _xs_{q_\theta}(x)) \right)dx +const., \label{eq:empirical:sm}
\end{talign}
where $\nabla_x s_{q_\theta}(x)=\nabla^2_x\log q_\theta(x)$ is the Hessian matrix and the constant represents the terms that are independent of $\theta$.
The integration over $p_d$ can be approximated by Monte-Carlo using $\mathcal{X}_{train}$. Because both $s_{q_\theta}$ and $\nabla_x s_{q_\theta}$ only depend on $f_\theta$, the normalizer $Z_q(\theta)$ is not required during training and we only need to estimate $Z_q(\theta^*)$ once after training.
Therefore, density estimation with FD in this setting contains two steps: (1) learn $\theta^*$ using Equation \ref{eq:empirical:sm}; (2) estimate $Z(\theta^*)$ to obtain the normalized density $q_\theta(x)$. This scheme can result in blindness in practice~\cite{wenliang2019learning}.

To heal the blindness, we can apply the proposed MFD. However, if we directly minimize MFD in step (1), the score $ s_{\tilde{q}_\theta}(x)=\nabla_x\log\left(\beta \exp(-f_\theta(x))/Z_q(\theta))+(1-\beta)m(x)\right)$
requires estimating $Z_q(\theta)$.
This negates the advantage of using score matching because now $Z_q(\theta)$ must be estimated for every gradient step during training (similar to MLE).
To avoid this, we propose to instead directly approximate $\tp_d$ with an energy-based model $\tq_\theta(x) \equiv e^{-f_\theta(x)}/Z_{\tilde{q}}(\theta)$ and
$\tq_\theta$ can then be trained using 
\begin{talign}
   \mathrm{FD}(\tilde{p}_d||\tilde{q}_\theta)=\frac{1}{2}\int_{\sR^d} \tilde{p}_d(x)\left( ||s_{\tilde{q}_\theta}(x)||_2^2+2 \Tr\left(\nabla _xs_{\tilde{q}_\theta}(x)\right) \right)dx +const.\label{eq:mixture:sm},
\end{talign}
where the integration over $\tilde{p}_d(x)$ can be approximated using the samples from the mixture $\tilde{p}_d(x)=\beta p_d(x) +(1-\beta)m(x)$.
 Therefore, the learning of $\theta$ is independent of $Z_{\tilde{q}}(\theta)$.
 Optimally we have $
    \tq_{\theta^*}(x) = \tp_d(x) = \beta p_d(x)+(1-\beta)m(x).$
To obtain a model of the underlying true density $q^* = p_d$, we need to remove the mixture component from $\tq_{\theta^*}$, which can be done through a `correction step':
\begin{talign}
q^{*}(x)= \frac{1}{\beta}\left(\tilde{q}_{\theta^*}(x)-(1-\beta)m(x)\right)=\frac{1}{\beta}\left(\frac{e^{-f_{\theta^*}(x)}}{Z_{\tilde{q}}(\theta^*)}-(1-\beta)m(x)\right).\label{eq:correction}
\end{talign}
This procedure for obtaining $q^*$ is equivalent to $q^*(x)=\arg\min_q \MFD(p_d(x)||q(x))$ and when $\MFD(p_d(x)||q(x))=0$, we have $q^*(x)=p_d(x)$. Therefore,
density estimation with MFD in this setting contains three steps: (1) learn $\theta^*$ by minimizing Equation \ref{eq:mixture:sm}; (2) 
estimate $Z_{\tq}(\theta^*)$; and (3)
apply the correction step (Equation \ref{eq:correction}) to obtain $q^*$. Compared to FD, the additional correction step has negligible computation cost. 
% \fxb{I think it is not super clear how the procedure above relates to MFD. I think we need a clear statement of which step above uses MFD.}\mt{what about now}

\textbf{Choice of m and $\beta$:} As we discussed in Section 3, a good $m$ should have support $\sR^d$ and be able to bridge disconnected component distributions. For a given set of data samples $\{x_1,\cdots,x_N\}\sim p_d$, we can simply choose $m(x) = \Normal(\bar\mu,\bar\Sigma)$, where $\bar\mu$ and $\bar\Sigma$ are the empirical mean and covariance of the available training data:
$\bar\mu=\frac{1}{N}\sum_{n=1}^N x_n,  \bar\Sigma=\frac{1}{N}\sum_{n=1}^N x_nx_n^T,$
which corresponds to an empirical moment matching approximation of  $p_d$ and can thus cover different components. The $\beta$ is treated as a hyper-parameter in our method. Intuitively, a large beta means that the proportion of data points from $p_d$ is small, and the model is learning $m$. On the other hand, a small value means we may still have the numerical version of the blindness issue. In this experiment, we use $\beta=0.8$ can find it can empirically heal blindness. We leave the theoretical study of choosing the $\beta$ into future work.

\textbf{Demonstration:}
We apply the proposed method to train a deep energy-based model and examine the performance against two target densities with multiple isolated components: 1) a weighted mixture of four Gaussians 
$p_d(x)= 0.1 g_1(x)+ 0.2 g_2(x)+ 0.3 g_3(x)+ 0.4 g_4(x)$, where $g_1, g_2,g_3,g_4$ are 2D Gaussians with identity covariance matrix and mean $[-5,-5],[-5,5],[5,5],[5,-5]$ respectively; and 2) a mixture of 3 concentric circles as proposed in \cite{wenliang2019learning}. We use Simpson's rule for the 2D numerical integration to estimate the normalization constant for both methods. The model specifications and training details can be found in Appendix~\ref{app:network}. In Figure 3 we plot the ground truth and the estimated density with classic FD and the proposed MFD methods. We also provide the corresponding $\KL$ evaluation (see Appendix~\ref{app:network}) between the ground truth density $p_d$ and the estimated model $p_\theta$. We find the proposed MFD method can significantly improve performance and heal the blindness problem. 

\begin{figure}[t]
\centering
\begin{subfigure}[b]{0.161\textwidth}
     \centering
\includegraphics[scale=0.3]{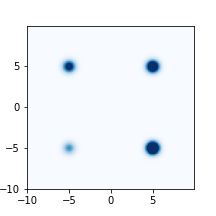}
    \caption{True}
\end{subfigure}
  \begin{subfigure}[b]{0.161\textwidth}
     \centering
\includegraphics[scale=0.3]{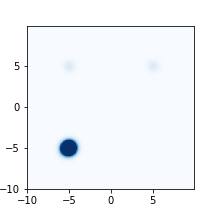}
    \caption{FD }
\end{subfigure}
  \begin{subfigure}[b]{0.161\textwidth}
     \centering
\includegraphics[scale=0.3]{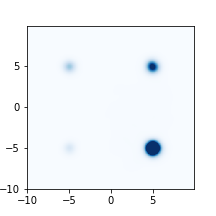}
    \caption{MFD}
\end{subfigure}
\begin{subfigure}[b]{0.161\textwidth}
     \centering
\includegraphics[scale=0.3]{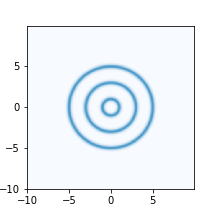}
    \caption{True}
\end{subfigure}
  \begin{subfigure}[b]{0.161\textwidth}
     \centering
\includegraphics[scale=0.3]{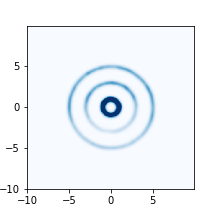}
    \caption{FD}
\end{subfigure}
  \begin{subfigure}[b]{0.161\textwidth}
     \centering
\includegraphics[scale=0.3]{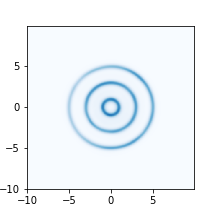}
    \caption{MFD}
\end{subfigure}
\caption{Density estimation comparisons with FD and MFD for the energy-based model. The $\mathrm{KL}(p_d||p_\theta)$ evaluations are 3.52/0.22 (b/e) for FD  and 0.17/0.01 (c/f) for MFD, lower is better.} 
% \ph{How many seeds did you do, is this an average?}\mt{one seed. It's fine for the workshop paper, need to do multiple runs in the aistats version.}
\vspace{-0.2cm}
\end{figure}

\section{Related Work}

% \fxb{I think the discussion below is not likely to be very exciting for readers which do not know about these various methods. Perhaps this content could be shortened into a "related work" subsection and each paragraph summarised as one or two lines max. The discussion could then contain some thoughts on how to extend our work (basically just everything we are planning to do for the full conference paper)} \mt{I found maybe people would be interesting in the alternative constructions like convolution or invertible transformation, do we also remove the appendix on this?}

In addition to the mixture construction, conducting a Gaussian convolution on both $p_d$ and $q_\theta$ can also bridge the disjoint components and defines a valid divergence~\cite{zhang2020spread}. However, the score function is generally intractable for a deep energy-based model $q$, see Appendix \ref{app:spread:FD} for a detailed discussion.
% Following the spread $f$-divergence~\cite{zhang2020spread}, one can also define a \emph{spread FD} as $\widetilde{\mathrm{FD}}_k(p(x)||q(x))\equiv \mathrm{FD}(\tilde{p}(\tilde{x})||\tilde{q}(\tilde{x}))$, where $\tilde{p}(\tilde{x})=\int p(x)k(\tilde{x}|x)dx$, $\tilde{q}(\tilde{x})=\int q(x)k(\tilde{x}|x)dx$ and $\widetilde{\mathrm{FD}}_k(p||q)=0 \Leftrightarrow p=q$, see the Appendix C in \cite{zhang2020spread} for a proof. 
% The convolution can heal the blindness  since $\tilde{p}$ and $\tilde{q}$ will have support $\sR^d$, but the score $\nabla_{\tilde{x}}\log \tilde{q}(\tilde{x})$ is generally intractable for a deep energy-based model $q$, 

% However, the \emph{spread FD} has a deep connection with denoising score matching~\cite{vincent2011connection} and diffusion models~\cite{song2021maximum}

% \ph{I think it would be better to move the points from the following 2 paragraphs into section 1 or 2, or an explicit related work section}\mt{can do that in the arxiv verion.}
Paper \cite{song2019generative} proposes to add  Gaussian noise with variance $\sigma^2$ \emph{only} to $p_d$ and anneal $\sigma^2\rightarrow 0$ during training. This helps alleviate the blindness problem in the early stage of training but when $\sigma^2\approx 0$, the blindness phenomenon will be observed again, see Appendix~\ref{app:data:noise} for an example. 

Paper~\cite{gong2021interpreting}  proposes to transform $p$ and $q$ with a common differentiable invertible function before defining the FD, which is also shown to be equivalent to \cite{barp2019minimum}. However, since the invertible transformation is a homeomorphism and will not change 
the topology of its domain \cite{cornish2020relaxing,zhang2021flow}, the invertible transformation will not fix the blindness caused by the disconnected support sets in principle. 

% $g$. Let $p(y)=\int \delta (y-g(x))p(x)dx$  and $q(y)=\int \delta (y-g(x))q(x)dx$, we have $\mathrm{FD}(p(y)||q(y))=0\Leftrightarrow p(x)=q(x)$. 

% \mt{likelihood-based training}

The blindness problem also exists in other score-based applications.
% that use the score function~\cite{wenliang2020blindness}, e.g. Goodness of Fit test with KSD\cite{chwialkowski2016kernel,liu2016kernelized,key2021composite,antonin2022ksd} or  Stein variational gradient descent (SVGD)~\cite{liu2016stein}. 
As discussed in Section 3, directly applying the MFD requires knowing the normalizer, which potentially sheds light on choosing score-based methods. We  leave the case-by-case study of how to heal the blindness to future work.
% \ph{i think you should elaborate what you mean by this}.

\clearpage
\newpage
\bibliographystyle{abbrvnat}
\bibliography{main.bib}

\clearpage
\newpage
\appendix

\section{Derivations and Proofs}
\subsection{Derivation of Equation \ref{eq:disjoint} \label{app:derivation}}
Let two differentiable densities $g_1$ and $g_2$ have disjoint supports  $\X_1\cap \X_2=\emptyset$ and
\begin{talign}
    p(x)=\alpha_p g_1(x)+(1-{\alpha_p})g_2(x), \quad q(x)={\alpha_q} g_1(x)+(1-{\alpha_q})g_2(x).
\end{talign}
The FD between $p$ and $q$ can be written as
\begin{talign}
  \mathrm{FD}(p||q)=\frac{\alpha_p}{2} \int_{\X_1} g_1(x)||s_ p(x)-s_q(x)||_2^2 dx  +\frac{1-\alpha_p}{2} \int_{\X_2} g_2(x)||s_p(x)-s_q(x)||_2^2 dx.
\end{talign}
Since $g_1$ and $g_2$ has disjoint support, so  $g_2$ will be a zero function on the support of $g_1$, so  $g_2(x')=\nabla_x g_2(x')=0$ for $x'\in \X_1$. We then have
\begin{talign}
   s_p(x')=\frac{\alpha_p\nabla g_1(x')+\cancel{(1-\alpha_p)\nabla g_2(x')}}{\alpha_p g_1(x')+\cancel{(1-\alpha)g_2(x')}}= \frac{\alpha_p \nabla_x g_1(x')}{\alpha_p g_1(x')}= s_{g_1}(x'),
\end{talign}
and
\begin{talign}
   s_q(x')=\frac{\alpha_q\nabla g_1(x')+\cancel{(1-\alpha_q)\nabla g_2(x')}}{\alpha_q g_1(x')+\cancel{(1-\alpha)g_2(x')}}= \frac{\alpha_q \nabla_x g_1(x')}{\alpha_q g_1(x')}=s_{g_1}(x'),
\end{talign}
% \fxb{You have some issues with indices in $p$ and $q$ above}
Similarly, for $x'\in \X_2$ we have $s_p(x')=s_q(x')=s_{g_2}(x')$.
Therefore, the FD is equivalent to 
% \fxb{Technically it is only independent of $\alpha_q$, but it clearly depends on $\alpha_p$ since we are integrating wrt $p$}.\mt{but in this example the score difference is 0 everywhere so also is independent of $\alpha_q$, but I agree only write $\alpha_p$ here will be more clear.}
\begin{talign}
\resizebox{0.9\hsize}{!}{$%
  \mathrm{FD}(p||q)=\frac{\alpha_p}{2} \int_{\X_1} g_1(x)||s_{g_1}(x)-s_{g_1}(x)||_2^2 dx   +\frac{1-\alpha_p}{2} \int_{\X_2} g_2(x)||s_{g_2}(x)-s_{g_2}(x)||_2^2 dx=0,$}
\end{talign}
which is independent of $\alpha_q$.
\subsection{Proof of Theorem~\ref{theorem:connected}\label{app:proof:connected}}
The following two lemmas can be found in \citet[Corollary 2.41 and Theorem 2.42]{folland2003advanced}. For completeness, we also provide simplified proofs. 

\begin{lemma}
    Suppose $f: \X\rightarrow \sR$ is differentiable on an open convex set $\X\subseteq \sR^d$ and $\nabla_x f(x)=0$ for all $x\in \X$, then $f$ is a constant on $\X$.\label{app:lemma:convex}
\end{lemma}
\begin{proof}
    For any two points $x_1,x_2\in \X$, we denote the the line segment that connects $a,b$ as $L_{x_1,x_2}$. Since $\X$ is a convex set, then $L_{x_1,x_2}\subseteq \X$. By the Mean Value Theorem (see \citet[Theorem 2.39]{folland2003advanced}), there exists a point $x_3\in L_{x_1,x_2}$ such that 
    $f(x_2)-f(x_1)=\nabla_x f(x_3) (x_2-x_1).$
    Since $x_3\in S$, so $\nabla_x f(x_3)=0$ thus $f(x_2)=f(x_1)$. Therefore, $f$ has to be a constant function.
\end{proof}

\begin{lemma}
    Suppose $f:\X\rightarrow \sR$ is differentiable on a connected open set $\X\subseteq\sR^d$ and $\nabla_x f(x)=0$ for all $x\in \X$, then $f$ is a constant on $\X$.\label{app:lemma:connected}
\end{lemma}
\begin{proof}
For any point $a\in \X$, we define $\X_1=\{x\in \X: f(x)=f(a)\}$ and $\X_2=\{x\in \X: f(x)\neq f(a)\}$, so $\X=\X_1\cup \X_2$ by construction. 
For every $x\in \X_1$, there is a ball $B\in S$ centred at $x$. Since $B$ is convex, we have $B\in \X_1$ by Lemma \ref{app:lemma:convex}. Therefore, every point $x\in \X_1$ is an interior point of $\X_1$, so $\X_1$ is an open set. The image of $\X_2$ under $f$: $\sR\setminus\{f(a)\}$ is an open set, so $\X_2$ is a open set since $f$ is a continuous function (see \citet[Theorem 1.33]{folland2003advanced}). We thus have both $\X_1$ and $\X_2$ are open sets and $\X_1$ is non-empty (it contains $a$). Since any connected space cannot be written as an union of two disjoint non-empty sets (see \citet[Definition 2.4.1]{tao2015analysis}), so $\X=\X_1\cup \X_2$  indicates $\X_2=\emptyset$. Therefore, $f$ is a constant function.
\end{proof}
We can then prove the Theorem \ref{theorem:connected}. For two a.c. distributions that are supported on a connected space $\X\subseteq\sR^d$ with differentiable density $p$ and $q$. Then $\FD(p||q)=0\Leftrightarrow \nabla_x\log p(x)=\nabla_x\log q(x)$ for $x\in S$. We define function $f(x)=\log p(x)-\log q(x)$, so $f(x)$ differentiable on $\X$ and $\nabla_x f(x)=0$. By Lemma \ref{app:lemma:connected}, we have $f$ as a constant function (we denote as $c$) so we have $p=q\exp(c)$. Since $p$ and $q$ are densities, we have $\int q(x)\exp (c)dx=1\Leftrightarrow c=0$. Therefore,
$\FD(p||q)=0\Leftrightarrow p=q$.

\subsection{Proof of Theorem \ref{theo:ill:defined} \label{app:ill:defined}}
Since we can always represents a distribution with disjoint support set as a mixture distribution with components supported on several connected subsets, we can then prove the theorem  by Proposition~\ref{prop:1}.
\begin{prop}[FD is ill-defined on disconnected sets]
Let a set of a.c. distributions have differentiable densities $\{g_1,\cdots,g_K\}$ with mutual disjoint (disconnected) support sets $\{\X_1,\cdots,\X_K\}$: $\X_i\bigcap\X_j=\emptyset$ for any $i\neq j$ and each support $\X_i$ is connected. Let two densities $p=\sum_{k} \alpha^k_p g_k$ and $q=\sum_{k} \alpha^k_q g_k$ with positive coefficients $\sum_{k} \alpha^k_p=1$ and $\sum_{k=1} \alpha^k_q=1$. Then $\FD(p||q)=0\Leftrightarrow \alpha_p^k= \alpha_q^ke^{c_k}$, where $\{c_1,\cdots,c_K\}$ is a set of constants with constraints $\sum_{k}e^{c_k}=1$.\label{prop:1}
\end{prop}
We can decompose 
$\FD(p||q)=\frac{1}{2}\sum_{k=1}^K \alpha_p^k \int_{\X_k} g_k(x)||s_p(x)-s_q(x)||_2^2dx$. Since $\alpha^k_p$ and $g_k$ are positive,   $\FD(p||q)=0\Rightarrow \int_{\X_k} g_k(x)||s_p(x)-s_q(x)||_2^2dx=0$ for any $k$, so $\nabla_x\log p(x)=\nabla_x\log q(x)$ for $x\in \bigcup_{k=1}^K {\X_k}$. Since ${\X_k}$ is connected, by Lemma \ref{app:lemma:connected}, we have for $x\in {\X_k}$, $\log p(x)-\log q(x)=c_k\Leftrightarrow p(x)=q(x)e^{c_k}\Leftrightarrow  \alpha_p^k g_k(x)=\alpha_q^k g_k(x) e^{c_k}\Leftrightarrow  \alpha_p^k= \alpha_q^ke^{c_k}$, where $\{c_1,\cdots,c_K\}$ is a set of constants. Since $\sum_k\alpha_p^k=\sum_k\alpha_q^k e^{c_k}=1$ and $\sum_{k}\alpha_q^k=1$, we then have the constrain
$\sum_{k} e^{c_k}=1$.

\subsection{Derivation of Score Matching\label{app:score:matching}} 
Let $p_d$ and $q_\theta$ are differentiable densities with a common support $\X\subseteq \sR^d$ and  assume $q_\theta$ is twice differentiable, we can  rewrite the FD as~\cite{hyvarinen2005estimation}
\begin{talign}
    \mathrm{FD}(p_d(x)||q_\theta(x))&=\frac{1}{2}\int_{\X} p_d(x)||s_{p_d}(x)-s_{q_\theta (x)}||^2_2dx\\
    &=\frac{1}{2}\int_{\X} p_d(x)\left( s^2_{p_d}(x)+s^2_{q_\theta}(x)-2 s_{p_d}(x)s_{q_\theta}(x) \right) dx\\
    &=\frac{1}{2}\int_{\X} p_d(x)\left( s^2_{q_\theta}(x)-2 s_{p_d}(x)s_{q_\theta}(x) \right)dx +const.,
\end{talign}
where the constant terms are independent of the model parameters $\theta$. Using the log-trick, we have 
\begin{talign}
    \int_{\X}  p_d(x)s_{p_d}(x)s_{p_\theta}(x)dx&= \int_{\X} \nabla_x p_d(x)s_{q_\theta}(x)dx.
\end{talign}
For simplicity, we assume $\X=\sR$ and $p_d(x)s_{p_\theta}(x)$ vanishes at $-\infty$ and $\infty$, using integration by parts, we have
\begin{talign}
    \int_{\X} \nabla_x p_d(x)s_{q_\theta}(x)dx&=\underbrace{p_d(x)s_{q_\theta}(x) \Big|^{+\infty}_{-\infty}}_{=0}-\int_{\X} p_d(x)\nabla_xs_{q_\theta}(x)dx.
\end{talign}
In general, this holds for $\X=\sR^d$ and $\lim_{||x||\rightarrow\infty} p_d(x)s_{q_\theta}(x)=0$ or $\X\subseteq\sR^d$ is a compact subset of $\sR^d$ and $f(x)p(x)=0$ for $x\in\partial\X$ where $\partial\X$ 
is the piecewise smooth boundary of $\X$ (by the divergence theorem~\cite[Theorem 5.34]{folland2003advanced}), also see~\cite{liu2016kernelized} for a similar discussion. Therefore, we have
\begin{talign}
     \mathrm{FD}(p_d(x)||q_\theta(x))&=\frac{1}{2}\int_{\X} p_d(x)\left( s^2_{q_\theta}(x)+2 \nabla _xs_{q_\theta}(x) \right)dx..
\end{talign}
% This can be extended for $p$ and $q$ are supported on a  bounded open set. For example, in the 1d case, assume $p$ and $q$ are supported on $(a,b)$ and $p_d(x)s_{q_\theta}(x)$ vanishes at a and b: $p_d(a^+)s_{q_\theta}(a^+)=p_d(b^-)s_{q_\theta}(b^-)=0$ when $a^+\rightarrow a$ and $b^-\rightarrow b$. In this case, we also have
% \begin{talign}
%     \int_a^b \nabla_x p_d(x)s_{q_\theta}(x)dx&=\underbrace{p_d(x)s_{q_\theta}(x) \Big|^{b}_{a}}_{=0}-\int p_d(x)\nabla_xs_{q_\theta}(x)dx.,
% \end{talign}
% and

\subsection{Kernelized Stein Discrepancy Extensions\label{app:ksd}}
For two a.c. distributions $p$ and $q$, the Kernelized Stein Discrepancy~\cite{liu2016kernelized,chwialkowski2016kernel} can be defined as (see \cite[Definition 3.2]{liu2016kernelized})
\begin{talign}
    \mathrm{KSD}(p||q)=\mathbb{E}_{x,x'\sim p} \left [(s_p(x)-s_q(x))k(x,x')(s_p(x')-s_q(x'))\right ],
\end{talign}
where $k$ is an integrally strictly positive kernel (see \cite[Definition 3.1]{liu2016kernelized}) and $x,x'$ are \iid samples from $p(x)$. The $\mathrm{KSD}(p||q)=0$ if and only if $s_p=s_q$ (see~\cite{liu2016kernelized,chwialkowski2016kernel}). Therefore, when $p$ and $q$ are supported on a connected open set, by Lemma \ref{app:lemma:connected}, we have $\mathrm{KSD}(p||q)=0\Leftrightarrow s_p=s_q\Leftrightarrow p=q$. When $p$ and $q$ are supported on a disconnected space, we have $\mathrm{KSD}(p||q)=0 \centernot\Rightarrow p=q$. This is because the KSD can be upper bounded by a (positively) scaled FD \cite[Theorem 5.1]{liu2016kernelized}:
\begin{talign}
    |\mathrm{KSD}(p||q)|\leq \sqrt{\mathbb{E}_{x,x'\sim p}[k(x,x')^2]}\times \FD(p||q),
\end{talign}
we then have $\FD(p||q)=0 \Rightarrow \mathrm{KSD}(p||q)=0$. When  $p$ and $q$ are supported on a disconnected space, we have $\mathrm{FD}(p||q)=0 \centernot\Rightarrow p=q$ (Theorem~\ref{theo:ill:defined}), so $\mathrm{KSD}(p||q)=0 \centernot\Rightarrow p=q$.

\subsection{Proof of Theorem \ref{theo:mdf} \label{app:theorem}}
Since the support of $m$ as $\X_m=\sR^d$ then $\tilde{p}$ and $\tilde{q}$ have the same support $\X=\sR^d$. For the score functions, we also have
\begin{talign}
    &\int_{\X}|| s_{\tilde{p}}(x)||_2^2\tilde{p}(x)dx=\int_{\X} \norm{\nabla_x\log (\beta p(x)+(1-\beta)m(x))}_2^2 \tilde{p}(x)dx\\
    =& \int_{\X} \norm{\frac{\beta\nabla_x p(x)+(1-\beta)\nabla_x m(x)}{\beta p(x)+(1-\beta)m(x)}}_2^2 \tilde{p}(x)dx \\
    \leq& \int_{\X} \norm{\frac{\beta \nabla_x p(x)}{\beta p(x)+(1-\beta)m(x)}}_2^2\tilde{p}(x)dx+\int_{\X} \norm{\frac{(1-\beta) \nabla_x m(x)}{\beta p(x)+(1-\beta)m(x)}}_2^2\tilde{p}(x)dx \\
    \leq &\int_{\X} ||s_p||_2^2\tilde{p}(x)dx+\int_{\X} ||s_m||_2^2\tilde{p}(x)dx \leq \int_{\X} ||s_p||_2^2p(x)dx+\int_{\X} ||s_m||_2^2p(x)dx<\infty,
\end{talign}
so $s_{\tilde{p}}\in L^2(\tilde{p})$ and similarly  $s_{\tilde{q}}\in L^2(\tilde{p})$. Therefore, the FD between $\tilde{p}$ and $\tilde{q}$ is a valid divergence i.e. $\FD(\tilde{p}||\tilde{q})=0\Leftrightarrow \tilde{p}=\tilde{q} \Leftrightarrow \beta  p(x)+(1-\beta) m(x)=\beta  q(x)+(1-\beta) m(x)
  \Leftrightarrow  p(x)=q(x)$, thus $\MFD(p||q)=0\Leftrightarrow \FD(\tilde{p}||\tilde{q})=0\Leftrightarrow p=q$.

\section{Experiment Details\label{app:network}}
For both experiments, we sample 100k data from $p_d$ as our training datasets. The energy network $f_\theta(x)$ is a 3-layer feedforward network with 200 hidden units and swish activation functions~\cite{ramachandran2017searching}. We train the model for 30k iterations with the Adam optimizer~\cite{kingma2014adam} and batch-size 300. 
For the numerical integration we use Simpson’s rule provided in the package~\cite{2020SciPy-NMeth}. We use a Monte-Carlo approximation to estimate the KL divergence evaluations $\widehat{\mathrm{KL}}(p_d(x)||q_\theta(x))= \frac{1}{K}\sum_{k=1}^K \log p_d(x_k)-\log p_\theta(x_k)$, where we use $K=10000$.

\section{Data Noise Annealing Doesn't Help\label{app:data:noise}}
In this section, we empirically show that only adding noise to the data and annealing the noise to $0$ during training won't fix the blindness problem in practice.
We use a deep energy-based model with a 3-layer feedforward neural network with 30 hidden units and tanh activation function to learn the toy mixture of two Gaussian distributions described in Section 1. We train the model with Adam optimizer with a learning rate $3e^{-4}$ for 10k iterations and batch size 300. We add convolutional Gaussian noise to the data samples with a standard deviation of 3.0 and anneal to 0 by multiplying by 0.9999 at each iteration. The noise at the end of training has a standard deviation less than 0.001. In Figure \ref{figure:toy:fd:1d:anneal} we plot the learned density during training. We find that when the noise is big the model can identify the correct mixture co-efficient, but when the noise is close to 0, the model fails to capture the correct mixing proportions. We also plot the density estimation results with vanilla FD and the proposed MFD in Figure \ref{figure:toy:fd:1d} and \ref{figure:toy:mfd:1d} and we find that the density estimation with MFD achieves the best performance.

\begin{figure}[H]
    \centering
\includegraphics[width=0.9\textwidth]{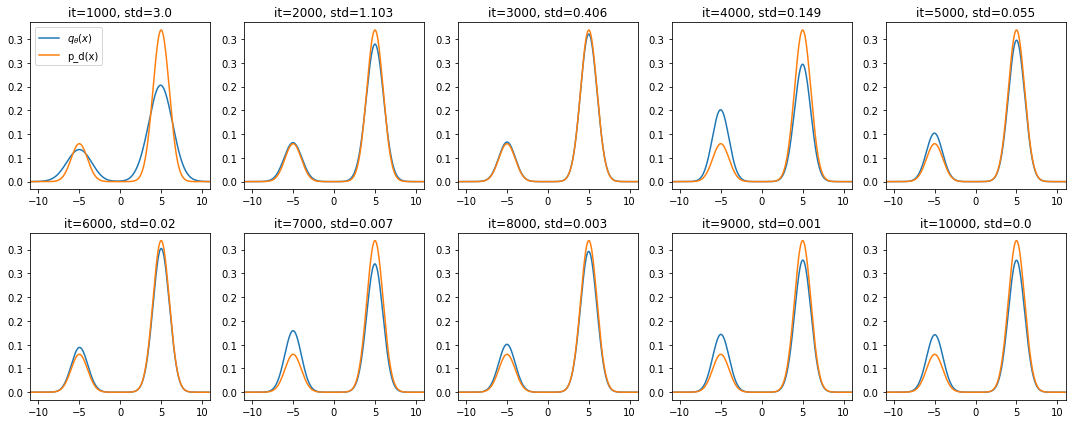}
    \caption{FD with training data noise annealing\label{figure:toy:fd:1d:anneal}}
\end{figure}
\begin{figure}[t]
\centering
\begin{subfigure}[b]{0.4\textwidth}
     \centering
\includegraphics[width=0.8\textwidth]{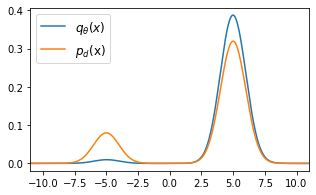}
    \caption{FD\label{figure:toy:fd:1d}}
\end{subfigure}
  \begin{subfigure}[b]{0.4\textwidth}
     \centering
\includegraphics[width=0.8\textwidth]{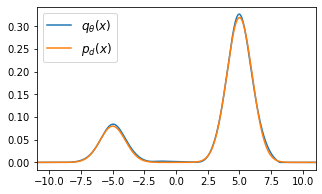}
    \caption{MFD\label{figure:toy:mfd:1d}}
\end{subfigure}

\caption{Density Estimation with FD and MFD.}

\end{figure}

\section{Spread Fisher Divergence\label{app:spread:FD}}
For two distributions with densities $p_d(x)$ and $q_\theta(x)$ with supports $\X_{p_d},\X_{q_\theta}\subseteq\sR^d$, we can choose $k(\tilde{x}|x)=\mathrm{N}(x,\sigma^2)$ and let 
\begin{talign}
\tilde{p}_d(\tilde{x})=\int_{\X_{p_d}} k(\tilde{x}|x)p_d(x)dx\quad \tilde{q}_\theta(\tilde{x})=\int_{\X_{q_\theta}} k(\tilde{x}|x)q_\theta(x)dx
\end{talign}
We follow the spread $f$-divergence~\cite{zhang2020spread} and define the \emph{Spread Fisher Divergence} ($\widetilde{\mathrm{FD}}$) as
\begin{talign}
    \widetilde{\mathrm{FD}}_k(p_d||q_\theta)\equiv\mathrm{FD}(\tilde{p}_d||\tilde{q}_\theta),
\end{talign}
The convolution transform makes $\tilde{p}_d$ and $\tilde{q}_\theta$ have support $\
X_{\tilde{p}_d}=X_{\tilde{q}_\theta}=\sR^d$ (which is a connected space) and  $\widetilde{\mathrm{FD}}_k(p_d||q_\theta)\equiv\mathrm{FD}(\tilde{p}_d||\tilde{q}_\theta)$ is a valid discrepancy, i.e. $\widetilde{\mathrm{FD}}_k(p_d||q_\theta)=0\Leftrightarrow \tilde{p}_d=\tilde{q}_\theta\Leftrightarrow p_d=q_\theta$.  The spread Fisher divergence is also well-defined for the singular distributions (distributions that are not a.c. w.r.t Lebesgue measure), see~\cite{zhang2020spread} for a detailed discussion.

Similar to the FD, we can rewrite the $\widetilde{\mathrm{FD}}$ as
\begin{talign}
    \widetilde{\mathrm{FD}}_k(p_d||p_\theta)&= \frac{1}{2}\int_{\R^d} \tilde{p}_d(\tilde{x})\norm{s_{\tilde{p}_d}({\tilde{x})} -  s_{\tilde{q}_\theta}(\tilde{x})}_2^2d\tilde{x}\\
    &= \frac{1}{2}\int _{\R^d}\tilde{p}_d(\tilde{x})\left( s^2_{\tilde{q}_\theta}(\tilde{x})+2 \nabla _{\tilde{x}}s_{\tilde{q}_\theta}(\tilde{x}) \right)d\tilde{x} +const.,
\end{talign}
where the constant terms are independent of the model parameters. For an energy-based model $q_\theta(x)=e^{-f_\theta(x)}/Z(\theta)$, the spread model $\tilde{q}_\theta(\tilde{x})=\frac{1}{Z(\theta)\sqrt{2\pi\sigma^2}}\int e^{-f_\theta(x)-\frac{1}{2\sigma^2}(\tilde{x}-x)^2} dx$ has an intractable score. Additionally, unlike the mixture construction, if we directly assume $\tilde{q}_\theta(\tilde{x})=e^{-f_\theta(\tilde{x})}/Z(\theta)$, the underlying `correct' model $q_\theta(x)$ can not be recovered from $\tilde{q}_\theta(\tilde{x})$ even if we know $Z(\theta)$. Therefore, the $\widetilde{\mathrm{FD}}$ is not directly applicable in this case.

\end{document}